\documentclass[10pt]{article}
\usepackage{microtype}
\usepackage{booktabs,paralist}
\usepackage{amsfonts,amsmath,amssymb,amsthm} 
\usepackage{nicefrac,microtype,physics,mathtools} 
\usepackage{color}
\usepackage{graphicx,subcaption}
\usepackage{thmtools,thm-restate}
\usepackage[normalem]{ulem}
\usepackage{float}
\usepackage{xfrac}
\usepackage{array}

\usepackage{times}
\usepackage[round]{natbib}
\usepackage[left=0.875in,right=0.875in,top=1in,bottom=1in]{geometry}
\usepackage[colorlinks,allcolors=blue]{hyperref}
\renewcommand{\cite}[1]{\citep{#1}}

\renewenvironment{abstract}
{\centerline{\large\bf Abstract}\vspace{0.7ex}%
  \bgroup\leftskip 20pt\rightskip 20pt\noindent}%
{\par\egroup\vskip 0.25ex}

\DeclareMathOperator{\diag}{\text{diag}}

\DeclareMathOperator*{\argmin}{\mathrm{argmin}}

\DeclareMathOperator*{\arcsinh}{\mathrm{arcsinh}}
\DeclarePairedDelimiter{\tri}{\langle}{\rangle}

\DeclarePairedDelimiter{\prn}{(}{)}
\DeclarePairedDelimiter{\nrm}{\|}{\|}
\DeclarePairedDelimiter{\brk}{[}{]}

\newcommand{\W}{\mathbf{W}}
\newcommand{\bfu}{\mathbf{u}}
\newcommand{\U}{\mathbf{U}}

\newcommand{\reparam}[0]{\phi}

\renewcommand{\c}{\mathcal}

\newcommand{\bR}{\mathbb{R}}
\newcommand{\Sn}{\mathcal{S}^d}
\newcommand{\R}{\mathbb{R}}

\renewcommand{\tilde}{\widetilde}

\newcommand{\floor}[1]{\left\lfloor#1\right\rfloor}

\newcommand{\innerprod}[2]{\langle #1,#2 \rangle}

\newcommand{\remove}[1]{}

\newtheorem*{assumption*}{Assumptions}
\newtheorem*{definition*}{Definition}
\newtheorem{theorem}{Theorem}

\newtheorem*{conjecture*}{Conjecture}

\newcommand{\cW}{\mathcal{W}}
\newcommand{\wmdk}[1]{{\mathbf{w}}_\text{MD}^{(#1)}}
\newcommand{\wngdk}[1]{{\mathbf{w}}_\text{NGD}^{(#1)}}
\newcommand{\w}{\mathbf{w}}
\newcommand{\wteta}{\w\big(\floor{t}_{\eta}\big)}
\newcommand{\tildewteta}{\tilde{\w}\big(\floor{t}_{\eta}\big)}

\newcommand{\ztnu}{z_{\floor{t}_{\nu}}}

\newcommand{\removed}[1]{}

\title{\rule{\linewidth}{1.5pt} \textbf{Mirrorless Mirror Descent: A Natural Derivation of Mirror Descent} \rule[8pt]{\linewidth}{1pt}\vspace{-5mm}}
\author{\normalsize
\begin{minipage}{0.3\textwidth}
\centering
\textbf{Suriya Gunasekar}\\
\small Microsoft Research at \\Redmond, WA USA
\small\url{suriya@ttic.edu} \\
\end{minipage}\hspace{0.02\textwidth}
\begin{minipage}{0.3\textwidth}
\centering
\textbf{Blake Woodworth}\\
\small Toyota Technological  Institute at \\Chicago, IL USA \\
\small\url{blake@ttic.edu}
\end{minipage}\hspace{0.02\textwidth}
\begin{minipage}{0.3\textwidth}
\centering
\textbf{Nathan Srebro}\\
\small Toyota Technological  Institute at \\Chicago, IL USA \\
\small\url{nati@ttic.edu}
\end{minipage}
}
\date{}

\begin{document}
\maketitle

\begin{abstract}
 We present a  primal only derivation of Mirror Descent as a ``partial'' discretization of gradient flow on a Riemannian manifold where the metric tensor is the Hessian of the Mirror Descent potential.  We contrast this discretization to Natural Gradient Descent, which is obtained by a ``full'' forward Euler discretization.   This view helps shed light on the relationship between the  methods and allows generalizing Mirror Descent to general Riemannian geometries, even when the metric tensor is {\em not} a Hessian, and thus there is no ``dual.''
\end{abstract}

\section{Introduction}
Mirror Descent \citep{nemirovsky1983problem,beck2003mirror} is an important template first-order optimization method for optimizing w.r.t.~a geometry specified by a strongly convex potential function.  It enjoys rigorous guarantees, and its stochastic and online variants are even optimal for certain learning settings \citep{srebro2011universality}.
As its name implies, Mirror Descent was derived, and is typically
described, in terms of performing gradient steps in the \emph{dual space} using a \emph{mirror map}: in each iteration, one maps the iterate to the dual space through a link function, performs an update there, and then mirrors the updates back to the primal space. Understanding Mirror Descent in this way requires explicitly discussing  the dual space or the link function.

In this paper we derive a direct ``primal'' understanding of Mirror Descent, and in order to do so, turn to Riemannian Gradient Flow.  The infinitesimal limit of Mirror Descent, where the stepsize is taken to zero, corresponds to a Riemannian Gradient Flow on a manifold with a metric tensor that is given by the Hessian of the potential function used by Mirror Descent (see Section~\ref{sec:md-disc}).  
The standard forward Euler discretization of this Riemannian Gradient Flow gives rise to the Natural Gradient Descent algorithm \citet{amari1998natural}.  Our main observation is that a ``partial'' discretization of the flow, where we discretize the optimization objective but not the metric tensor specifying the geometry,  gives rise precisely to Mirror Descent (see Section~\ref{sec:md-disc}).  This view allows us to understand how Mirror Descent is, in a sense, more ``faithful'' to the geometry compared to Natural Gradient Descent.

The relationship we reveal between Mirror Descent
and Natural Gradient Descent is different from, and complementary to, the relationship discussed by \citet{raskutti2015information}---while their work showed how Mirror Descent and Natural Gradient Descent are \emph{dual} to each other, in the sense that Mirror Descent is equivalent to Natural Gradient Descent in the dual space, we avoid the duality altogether. We work {\em only} in the primal space, derive Mirror Descent directly, without considering the dual or link functions, and show how both methods are \textit{different  discretizations} of the \emph{same} flow  (more details in Section~\ref{sec:contrast-with-prior}).

As a consequence, our derivation of Mirror Descent allows us  to conceptually generalize Mirror Descent to any Riemannian manifold, including situations where metric tensor is {\em not} specified by the Hessian of any potential, and so there is no dual, no link function and no Bregman divergence.

\subsection{Background: Mirror Descent} \label{sec:md}
Consider optimizing a smooth objective $F:\cW\rightarrow\bR$ over a closed convex set $\cW\subseteq\bR^d$, $\min_{\w \in \c{W}} F(\w)$. We will focus on unconstrained optimization, \textit{i.e.,} $\c{W}=\bR^d$. 

Mirror Descent is a template first-order optimization algorithm specified by a strictly convex potential function $\psi:\c{W}\to\bR$. Mirror Descent was developed as a generalization of gradient descent to non-Euclidean geometries, where the local geometry is specified by the Bregman divergences w.r.t. $\psi$ given by $D_\psi(\w,\w')=\psi(\w)-\psi(\w')-\innerprod{\nabla\psi(\w')}{\w-\w'}$. The iterative updates of Mirror Descent with stepsize $\eta$ are defined as:
\begin{equation}
   \wmdk{k+1}\!=\!\argmin_{\w\in\cW}\,\eta\innerprod{\w}{\nabla F(\wmdk{k})}\!+\!D_\psi(\w,\wmdk{k}).
\label{eq:md-gen-1}
\end{equation}
For unconstrained optimization, the updates in \eqref{eq:md-gen-1} are equivalently given by:
\begin{equation}
    \nabla\psi(\wmdk{k+1})=\nabla\psi(\wmdk{k})-\eta \nabla F(\wmdk{k}),
\label{eq:md-gen}
\end{equation}
where $\nabla\psi$ is called the \emph{link function} and provides
a mapping between the primal optimization space $\w\in\cW$ and the dual
space of gradients.  Mirror Descent thus performs gradient updates in the dual ``mirror''.

\subsection{Background: Riemannian Gradient Flow}

Let $(\cW,H)$ denote a Riemannian manifold over $\cW=\bR^d$ equipped with a metric tensor $H(\w)$ at each point $\w\in\cW$. 
The metric tensor $H(\w):T_{\c{W}}(\w) \times T_{\c{W}}(\w)\to\bR$ denotes a smoothly varying \textit{local} inner product on the tangent space at $\w$. Intuitively, the tangent space $T_{\c{W}}(\w)$ is the vector space of all infinitesimal directions $\dd{\w}$ that we can move in while following a smooth path on $\c{W}$ \citep[for more detailed exposition see, e.g.][]{do2016differential}. For manifolds over $\R^d$, we can take the tangent space as $T_{\c{W}}(\w) = \R^d$ and the metric tensors can be identified with positive definite matrices $H(\w)\in\Sn_{++}$ that define local distances at $\w$ as  ${d}(\w,\w+\dd\w)=\sqrt{\dd\w^\top H(\w) \dd\w}$ for infinitesimal $\dd\w$.

The Riemannian Gradient Flow dynamics $\w(t)$ for the optimization problem $\min_{\w \in \c{W}} F(\w)$ with initialization $\w(0)=\w_\text{init}$ are obtained by seeking an infinitesimal change in $\w(t)$ that would lead to the best improvement in objective value, while controlling the length of the change in terms of the manifold geometry, that is,
\begin{equation}\label{eq:true-rgf}
\w(t+\dd t) = \argmin_{\w}\; F(\w)\dd t + \frac{1}{2} d(\w,\w(t))^2.
\end{equation}
For infinitesimal $\dd t$, using $\dd\w(t) = \w(t+\dd t)-\w(t)$, we can replace $F(\w)$ and $d(\w,\w(t))$ with their first order approximations\footnote{We use $\langle.,.\rangle $ to denote the canonical inner product in $\bR^d$ and $\nabla$ denotes the gradient operator such that $\langle \nabla F(\w),\dd\w\rangle=F(\w+\dd\w)-F(\w)$ for all infinitesimal $\dd\w$} $F(\w(t))+\langle \dd\w , \nabla F(\w(t)) \rangle$, and $d(\w,\w(t))=\sqrt{\dd\w^\top H(\w(t))\dd\w}$:
\begin{equation}
\label{eq:inf-rgf}
\dd\w(t) = \argmin_{\dd\w}\; \langle \dd\w , \nabla F(\w(t)) \rangle\dd t + \frac{1}{2}\dd\w^\top H(\w(t))\dd\w.
\end{equation}
Solving for $\dd \w$, we obtain:
\begin{equation}
    \dot{{\w}}(t) = -H(\w(t))^{-1}\nabla F(\w(t)),
    \label{eq:rgf}\tag{GF}
\end{equation}
where here and throughout we denote $\dot{{\w}}=\dv{\w}{t}$.  

We refer to the path specified by \eqref{eq:rgf} and initial condition $\w(0)=\w_\text{init}$ as \emph{Riemannian Gradient Flow} or sometimes simply as \textit{gradient flow}. 

Examples of Riemannian metrics and corresponding gradient flow that arise in learning and related areas:  
\begin{compactenum}
\item The standard Euclidean geometry is recovered with $H(\w)=I$.  In this case \eqref{eq:rgf} reduces to the standard gradient flow  $\dot{\w}=-\nabla F(\w)$.  When $H(\w)=H$ is fixed to some other positive definite $H$, we get the pre-conditioned gradient flow  $\dot{\w}=-H^{-1} \nabla F(\w)$, which can also be thought of as the gradient flow dynamics on a reparametrization $\tilde{\w}=H^{1/2} \w$, i.e.~with respect to geometry specified by a linear distortion.
\item For any strongly convex potential function $\psi$ over $\c{W}$, the Hessian $\nabla^2\psi$ defines a non-Euclidean metric tensor.  Examples include squared $\ell_p$ norms $\psi(\w)=\norm{\w}_p^2$ for $1<p\leq2$ ($p=2$ again recovers the standard Euclidean geometry) and a particularly important example is the simplex, endowed with an entropy potential $\psi(\w)=-\sum_i \w_i \log \w_i$.
\item Information geometry \citep{amari2012differential} is concerned with a manifold of probability distributions, e.g.~in a parametric family $\{p(x;\theta):\theta\in{\Theta}\subseteq\bR^d\}$, typically endowed with the metric derived from the KL-divergence.  In our notation, we would consider this as defining a Riemannian metric structure over the manifold parameters $\c{W}=\Theta$, with a metric tensor given by the Fisher information matrix $H(\theta)=\c{I}(\theta)=\mathbb{E}_x[-\nabla_\theta \log(p(x;\theta))\nabla_\theta \log(p(x;\theta))^\top\,|\,\theta]$.  Such a geometry can also be obtained by considering the entropy as a potential function and taking its Hessian.
\end{compactenum}

\section{Discretizing Riemannian Gradient Flow}\label{sec:discretization}
The Riemannian Gradient Flow is a continuous object defined in terms of a differential equation \eqref{eq:rgf}.  To utilize it algorithmically, we consider discretizations of the flow.

\subsection{Natural Gradient Descent}\label{sec:discNGD}
Natural Gradient Descent is obtained as  the  forward Euler discretization with stepsize $\eta$ of the gradient flow \eqref{eq:rgf}:
\begin{equation}
    \begin{gathered}
    \wngdk{k+1} = \wngdk{k}-\eta H(\wngdk{k})^{-1}\nabla F(\wngdk{k}), \\ \text{ where }\wngdk{0}=\w_\text{init}. 
    \label{eq:ngd}
\end{gathered}
\end{equation}
These Natural Gradient
Descent updates were suggested and popularized by
\citet{amari1998natural}, particularly in the context of
\emph{information geometry}, where the metric tensor is given by the
Fisher information matrix of some family of distributions. 

An equivalent way to view the updates \eqref{eq:ngd} is by discretizing the right-hand-side of the differential equation \eqref{eq:rgf} as follows:
\begin{equation}
    \dot{\w}(t) = -H\left(\wteta\right)^{-1}\nabla F\left(\wteta\right),
    \label{eq:ngd-disc}\tag{NGD}
\end{equation}
where $\floor{t}_{\eta}:=\eta\floor{\sfrac{t}{\eta}}$ denotes a discretization at scale $\eta$, i.e.~the largest $t'<t$ such that $t'$ is an integer multiple of $\eta$. 

 The differential equation \eqref{eq:ngd-disc} specifies a piecewise linear solution $\w(t)$ that interpolates between the Natural Gradient Descent iterates.  In particular, the Natural Gradient Descent iterates in \eqref{eq:ngd} are given by $\wngdk{k}=\w(\eta k)$ where $\w(t)$ is the solution of \eqref{eq:ngd-disc} with the initial condition $\w(0)=\w_{\text{init}}$, and we could have alternatively defined the forward Euler discretization in this way.

\subsection{Mirror Descent}\label{sec:md-disc}
In \eqref{eq:ngd-disc} we fully discretized the Riemannian Gradient Flow 
\eqref{eq:rgf}.  Now consider an alternate, partial, forward discretization of 
\eqref{eq:rgf}, where we discretize the gradient $\nabla F(\w)$, but
not the local metric $H(\w)$:
\begin{equation}
    \dot{\w}(t) = -H(\w(t))^{-1}\nabla F(\wteta).
    \label{eq:md-disc}\tag{MD}
\end{equation}
The resulting solution $\w(t)$ is piecewise smooth. We will again consider the sequence of iterates  at discrete points $t=\eta k$:
\begin{equation}
  \label{eq:md-as-disc}
  \w^{(k)} := \w(\eta k).
\end{equation}
Our main result is that if $H({\w}) = \nabla^2 \psi({\w})$, then the updates \eqref{eq:md-as-disc} from the solution of \eqref{eq:md-disc} are precisely the Mirror Descent updates in \eqref{eq:md-gen} with potential $\psi$.

\begin{theorem} \label{thm:main} Let $\psi:\cW\to\bR$ be  strictly convex and twice differentiable and let
  $H(\w)=\nabla^2 \psi(\w)$ be invertible everywhere. Consider the updates $\w^{(k)}=\w(k\eta)$ obtained from the solution of
  \eqref{eq:md-disc} with  initial condition $\w(0)$ and stepsize $\eta$.  Then $\w^{(k)}$ are the same as the Mirror Descent updates $\wmdk{k}$  in  \eqref{eq:md-gen} obtained with   $\psi$ as the potential function and the same initialization and stepsize.
\end{theorem}
\begin{proof}
Consider the Mirror Descent iterates with step size $\eta$ for link function $\nabla \psi$ from \eqref{eq:md-gen} \[\nabla\psi(\wmdk{k+1})=\nabla\psi(\wmdk{k})-\eta\nabla F(\wmdk{k}).\]

Define a Mirror Descent path $\hat{\w}(t)$ by linearly interpolating in the dual space as follows: 
\begin{equation*}
    \forall k,\forall t\in[k\eta,(k+1)\eta):\;
    \nabla\psi(\hat{\w}(t))=\nabla\psi(\wmdk{k})-(t-k\eta)\nabla F(\wmdk{k}). 
\end{equation*}

One can easily check that $\wmdk{k}=\hat{\w}(\eta k)$. The above equation describing a piecewise smooth path $\hat{\w}(t)$ equivalently corresponds to, $\dv{\nabla\psi(\hat\w(t))}{t}=-\nabla F\big(\hat{\w}\big(\floor{t}_\eta\big)\big)$.
Using the chain rule, we see that $\hat\w(t)$ follows the  path in \eqref{eq:md-disc}:
\begin{align*}
    \dot{\hat\w}(t)&=-\nabla^2 \psi(\hat\w(t))^{-1}\nabla F\big(\hat{\w}\big(\floor{t}_\eta\big)\big) \\
    &=-H(\hat\w(t))^{-1}\nabla F\big(\hat{\w}\big(\floor{t}_\eta\big)\big).
\end{align*}
This completes the proof of the theorem. 
\end{proof}

\subsection{A More Faithful Discretization?} Comparing the two discretizations \eqref{eq:ngd-disc} and \eqref{eq:md-disc} allows us to understand the relationship between Natural Gradient updates \eqref{eq:ngd} and Mirror Descent updates \eqref{eq:md-gen}.  Although both updates have the same infinitesimal limit, as previously discussed by, e.g.~\citet{gunasekar2018characterizing}, they differ in how the discretization is done with finite stepsizes: while Natural Gradient Descent corresponds to discretizing both the objective \textit{and} the geometry, Mirror Descent involves discretizing only the objective (accessed via $\nabla F({\w}(t))$), but not the geometry (specified by $H(\w(t))$).  
In this sense,  Mirror Descent is a ``more accurate'' discretization, being more faithful to the geometry of the search space.  

 This view of Mirror Descent also allows us to contrast the computational aspects of implementing the two algorithms. While both the algorithms are first order methods, which only require gradient access to the objective (at discrete iterates, $\nabla F(\w^{(k)})$, Natural Gradient Descent can be  implemented if we can compute the inverse Hessian of the metric tensor (\textit{i.e.,}~we need to either obtain and invert the metric tensor, or have direct access to its inverse).  But at least for a traditional implementation of the Mirror Descent updates in \eqref{eq:md-gen}, we need (a) the metric tensor $H(\w)$ to be a Hessian map, i.e.~the differential equation $H(\w)=\nabla^2 \psi(\w)$ should have a solution, and (b) we need to be able to efficiently calculate the link $\nabla\psi$ and inverse link $\nabla \psi^{-1}$ functions. More generally, one needs some way of solving the ordinary differential equation \eqref{eq:md-disc} to implement Mirror Descent. 

\subsection{When Does a Potential Exist?}  

In Theorem \ref{thm:main} we established that for any smooth strictly convex potential $\psi({\w})$ with everywhere invertible Hessian, we can define a metric tensor $H(\w)=\nabla^2 \psi(\w)$ so that Mirror Descent is obtained as a discretization of the Riemannian Gradient Flow \eqref{eq:rgf}. One might ask whether this connection with Mirror Descent holds for {\em any} Riemannian Gradient Flow.

The Riemannian Gradient Flow \eqref{eq:rgf}, and hence also the discretization \eqref{eq:md-disc}, can be defined for any smooth, invertible Riemannian metric tensor $H:\bR^d\rightarrow\Sn_{++}$.  But the classic Mirror Descent updates \eqref{eq:md-gen} are only defined in terms of a potential function $\psi$.  To relate Riemannian Gradient Flow w.r.t.~some metric tensor $H(\w)$ to Mirror Descent, we need to identify such a potential function, i.e.~a function $\psi:\c{W}\rightarrow\R$ s.t.~$\nabla^2 \psi=H$.  That is, we need there to be a solution to the partial differential equation $\nabla^2 \psi=H$, or in other words for $H$ to be a Hessian map.

When is a metric tensor a Hessian map? Or, when is there a solution to $\nabla^2 \psi=H$?
By Poincar\'e's lemma, the rows of the metric tensor, $H_i:\bR^d\rightarrow\bR^d$  are gradients (\textit{i.e.,} $\forall i, H_i=\nabla \phi_i$ for some $\phi_i:\bR^d\to\bR$) if and only if they satisfy the following symmetry condition:
\begin{equation}\label{eq:Hijk}
\forall_{{\w}\in\cW}\,\forall_{i,j,k\in\{1,\ldots,n\}}\quad    \frac{\partial H_{i,j}({\w})}{\partial {\w}_k} = \frac{\partial H_{i,k}({\w})}{\partial {\w}_j}.
\end{equation}
Thus, \eqref{eq:Hijk} is equivalent to $H$ being a Jacobian of some vector-valued function $\phi=[\phi_1,\phi_2,\ldots,\phi_d]:\bR^d\to\bR^d$. Further, since $H$ is symmetric (by definition), $\phi$ also satisfies the same symmetry condition $\frac{\partial \phi_i({\w})}{\partial {\w}_j} = H_{i,j}(\w) = H_{j,i}(\w) = \frac{\partial \phi_j({\w})}{\partial {\w}_i}$, and hence in turn is a gradient field (\textit{i.e.,} $\phi=\nabla\psi$ for some $\psi$). 
Therefore, \eqref{eq:Hijk} is equivalent to $\nabla^2 \psi=H$ having a solution, and since $H({\w})$ is positive definite (in order to be a valid metric tensor), $\psi$ must be strictly convex, as we would desire of a potential function.  Hence, we can conclude that the discretization of Riemannian Gradient Flow in \eqref{eq:md-disc} corresponds to classic Mirror Descent \eqref{eq:md-gen} for some strictly convex potential $\psi$ if and only if \eqref{eq:Hijk} holds. The requirement in \eqref{eq:Hijk} is non-trivial and does not hold in general, for instance, a seemingly simple metric tensor $H(\w) = I + \w\w^\top$ fails to satisfy \eqref{eq:Hijk} and therefore is not a Hessian map\footnote{This metric tensor $H(\w) = I + \w\w^\top$ can arise by considering the $d$-dimensional manifold embedded in $\R^{d+1}$ which consists of points $\begin{bmatrix}\w \\ \frac{1}{2}\nrm*{\w}^2\end{bmatrix}$ with local distances induced by the Euclidean geometry on $\R^{d+1}$. The distance between $\w$ and $\w + \dd \w$ is then given by
\begin{multline*}
d(\w, \w+\dd\w) 
= \nrm*{\begin{bmatrix}\w \\ \frac{1}{2}\nrm*{\w}^2\end{bmatrix} - \begin{bmatrix}\w+\dd\w \\ \frac{1}{2}\nrm*{\w+\dd\w}^2\end{bmatrix}} \\
= \sqrt{\dd\w^\top (I + \w\w^\top) \dd\w + \frac{1}{4}\nrm{\dd\w}^4 - \tri{\w,\dd\w}\nrm{\dd\w}^2}
\end{multline*}
For infinitesimal $\dd\w$, this means $d(\w, \w+\dd\w) = \sqrt{\dd\w^\top (I + \w\w^\top) \dd\w}$, and indeed the metric tensor is described by $H(\w) = I + \w\w^\top$. That this is not a Hessian map can be seen by simply calculating $\frac{\partial}{\partial \w_1} H_{1,2}(\w) = \frac{\partial}{\partial \w_1} \w_1\w_2 = \w_2 \neq \frac{\partial}{\partial \w_2} H_{1,1}(\w) = \frac{\partial}{\partial \w_2} 1 + \w_1^2 = 0$.}.

\subsection{Contrast With Prior Derivations} \label{sec:contrast-with-prior}
We emphasize that our derivation of Mirror Descent as a partial discretization of Riemannian Gradient Flow is rather different from, and complementary to, a previous relationship pointed out between Natural Gradient Descent and Mirror Descent by \citet{raskutti2015information}.  Their derivation {\em does} rely on duality and existence of a potential, and thus a link function.   \citeauthor{raskutti2015information} showed that Mirror Descent over $\w\in\c{W}$  corresponds to Natural Gradient Descent {\em in the dual}, that is after a change of parametrization given by the link function $\tilde{\w}= \nabla\psi(\w)$. \removed{Thus, while Natural Gradient Descent takes steps along straight rays in the primal space, they discuss how Mirror Descent takes steps along straight rays in the dual.}

In terms of the discretized differential equations \eqref{eq:ngd-disc} and \eqref{eq:md-disc}, the above relationship can be stated as follows: the path of the Natural Gradient Descent discretization \eqref{eq:ngd-disc} is piecewise linear in the primal space, i.e.~$\w(t)$ is piecewise linear, while the path of the Mirror Descent discretization \eqref{eq:md-disc} is piecewise linear in the \textit{dual space}, i.e.~$\nabla\psi(\w(t))$ is piecewise linear, and consequently curved in the primal space. But this view does not explain why Mirror Descent might be preferable when we are interested in the primal geometry.  In contrast, here we focus only on the (primal) Riemannian geometry, do not use a link function nor the dual, and highlight why Mirror Descent is more faithful to this primal geometry.  

\citeauthor{raskutti2015information}'s dual view is also captured by another popular way of developing Mirror Descent as a discretization of a differential equation: when the metric tensor is a Hessian map and $H(\w)=\nabla^2 \psi(\w)$, then the partial differential equation \eqref{eq:rgf} is equivalent to the following\footnote{To see this, apply the chain rule to the left hand side of \eqref{eq:grfdual} to get $\nabla^2\psi(\w(t))\dot\w(t)=-\nabla F(\w(t))$ and then multiply both sides by $\nabla^{2}\psi(\w(t))^{-1}=H(\w(t))^{-1}$ to get \eqref{eq:rgf}}  \citep{nemirovsky1983problem,warmuth1997continuous,raginsky2012continuous}:
\begin{equation}\label{eq:grfdual}
\dv{}{t}{ \nabla \psi(\w(t))} = -\nabla F(\w(t)).
\end{equation}
A forward Euler discretization of the differential equation \eqref{eq:grfdual} yields the Mirror Descent updates in \eqref{eq:md-gen}.  This can be viewed as using standard (full discretization) forward Euler, corresponding to piecewise linear updates and Natural Gradient Descent, but on the {\em dual} variables, i.e.~discretizing $\dv{\tilde{\w}}{t}$ where $\tilde{\w}= \nabla\psi(\w)$.

Viewing Mirror Descent as a forward Euler discretization of \eqref{eq:grfdual}, or as dual to Natural Gradient Descent, as in previous derivations and discussions of Mirror Descent still depends on having a link function $\nabla \psi(\w)$ such that the metric tensor is a Hessian map $H=\nabla^2 \psi$.   One might ask whether we could perform such derivations relying on a change-of-variables ``link'' function even if the metric tensor $H$ is not a Hessian map.  In other words, could we have a function $\reparam:\R^d\rightarrow\R^d$ such that $\reparam(\w(t))$ is piecewise linear under the Mirror Descent dynamics \eqref{eq:md-disc}, in which case we could derive Mirror Descent as Natural Gradient Descent on $\reparam(\w(t))$ or as the forward Euler discretization
\begin{equation}\label{eq:speculative-g}
\frac{\dd}{\dd t}\reparam(\w(t)) = -\nabla F(\wteta).    
\end{equation}
That is, when does there exists  $\reparam:\R^d\rightarrow\R^d$ such that \eqref{eq:speculative-g} is equivalent to \eqref{eq:md-disc} for any smooth objective $F(\w)$?  Applying the chain rule to the left hand side of \eqref{eq:speculative-g}, this would require that $\nabla \reparam = H$, which further implies that $H$ is a Hessian map\footnote{Applying the chain rule on \eqref{eq:speculative-g} and substituting \eqref{eq:md-disc} we get $\nabla \reparam(\w(t)) H(\w(t))^{-1} \nabla F(\wteta) = \nabla F(\wteta)$.  If this holds for any objective $F(\w)$, it must be that $\nabla \reparam H^{-1} = I$.  But $\nabla \reparam = H$  indicates that $H$ is a Jacobian map, which for symmetric $H$  further implies $\eqref{eq:Hijk}$ since both sides evaluate to $\partial_{j,k} \reparam_i$.}.
Therefore, if the metric tensor is {\em not} a Hessian map, there is no analogue to the link function, and we cannot obtain the discretization in \eqref{eq:md-disc} as Natural Gradient Descent after some change of variables, nor as a forward Euler discretization of a differential equation similar to \eqref{eq:grfdual}.

An distinguishing feature of our novel derivation of Mirror Descent that clearly differentiates it from all prior derivations, is that it does {\em not} require a potential, dual or link function, and so it does {\em not} rely on the metric tensor being a Hessian map.  This is exemplified by the fact that, unlike any prior derivation, it allows us to conceptually generalize Mirror Descent to metric tensors that are {\em not} Hessian maps.  

\section{Potential-free Mirror Descent}\label{sec:mirrorless}

As we emphasized, a significant difference between our derivation and previous, or ``classical'', derivations of Mirror Descent, is that our derivations did not involve, or even rely on the existence of potential function---that is, it did not rely on the metric tensor being a Hessian map.  If the metric tensor is {\em not} a Hessian map, we cannot define the link function nor Bregman divergence, and the standard Mirror Descent updates \eqref{eq:md-gen-1}--\eqref{eq:md-gen} are not defined, nor are any prior derivations of Mirror Descent that we are aware of.  Nevertheless our equivalent primal-only derivation of Mirror Descent \eqref{eq:md-as-disc} \emph{does} allow us to  generalize Mirror Descent as a first order optimization procedure to {\em any} metric tensor, even if it is not a Hessian map---we simply use \eqref{eq:md-as-disc} as the definition of Mirror Descent.

To be more precise, we can define $\wmdk{k}$ iteratively as follows:  given $\wmdk{k}$ and the gradient $g^{(k)}=\nabla F(\wmdk{k})$, consider the path defined by the differential equation 
\begin{equation}\label{eq:mirrorless}
\begin{gathered}
    \dot{\w}(t) = -H(\w(t))^{-1} g^{(k)} \textrm{ with } \w(0)=\wmdk{k},\\
    \text{ and let }\wmdk{k+1}=\w(\eta). 
\end{gathered}
\end{equation}

For a general metric tensor $H$, the above updates requires computing the solution of a differential equation at each step, which may or may not be efficiently computable (just as the standard Mirror Descent updates may or may not be efficiently computable depending on the link function).  Nevertheless, it is important to note that the differential equation \eqref{eq:mirrorless}  depends on the objective $F$ only through a single gradient $g^{(k)}=\nabla F(\wmdk{k})$.  That is, the only required access to the objective in order to implement the method is a single gradient access per iteration---the rest is just computation in terms of the pre-specified geometry (similar to computing the link and inverse link in standard Mirror Descent).  The updates \eqref{eq:mirrorless} thus define a valid first order optimization method, and independent of the tractability of solving the differential equation, could be of interest in studying optimization with first order oracle access under general geometries.

We also show in Appendix \ref{app:convergence} that when the eigenvalues of $H(\w)$ are bounded from above and below, and when the objective $F$ is smooth and strongly convex with respect to L2, then the updates \eqref{eq:mirrorless} guarantee linear convergence to a minimizer of $F$, even when the metric tensor is not a Hessian.
While this result is limited to the L2 geometry on $\mathcal{W}$, it at least suggests that the algorithm can be expected to succeed without relying on $H$ being a Hessian map.


\section{Importance of the Parametrization}\label{sec:reparam}

Our development in Section \ref{sec:discretization} relied not only on a Riemannian manifold $(\cW,H)$, but on a specific parametrization (or ``chart'') for the manifold, or in our presentation, on identifying the manifold $\cW$, and its tangent space, with $\bR^d$.  Let us consider now the effect of a change of parametrization (i.e.~on using a different chart).

Consider a change of parameters $\tilde{\w}=\reparam(\w)$ for some smooth invertible $\reparam$ with invertible Jacobian $\nabla \reparam$, that specifies an isometric Riemannian manifold $(\tilde{\cW},\tilde{H})$, i.e.,~such that  $d_H(\w,\w+\dd\w)={d}_{\tilde{H}}(\reparam(\w),\reparam(\w+\dd\w))$ for all infinitesimal $\dd\w$.  The metric tensor $\tilde{H}(\tilde{\w})$ for the isometric manifold   is given by 
\begin{equation}\label{eq:reparam-metric-tensor}
    \tilde{H}(\tilde{\w})=\nabla \reparam^{-1}(\tilde\w)^\top H(\reparam^{-1}(\tilde\w))\nabla \reparam^{-1}(\tilde\w),
\end{equation}
where  recall that the Jacobian of the inverse is the inverse Jacobian, $\nabla \reparam(\w)^{-1}=\nabla \reparam^{-1}(\tilde{\w})$.  This can also be thought of as using a different chart for the manifold (in our case, a global chart, since the manifold is isomorphic to $\bR^d$). 

In understanding methods operating on a manifold, it is important to separate what is intrinsic to the manifold and its geometry, and what aspects of the method are affected by the parametrization, especially since one might desire ``intrinsic'' methods that depend only on the manifold and its geometry, but not on the parametrization.  We therefore ask how changing the parametrization affects our development. In particular, does the Mirror Descent discretization, and with it the Mirror Descent updates change with parameterization? 

Consider  minimizing $F(\w)$, which after the reparametrization we denote as $\tilde{F}(\tilde\w)=F(\reparam^{-1}(\tilde\w))$.  The Riemannian Gradient Flow on $\tilde{\w}$ is 
\begin{equation}\label{eq:tilde-rgf}
\dot{\tilde{\w}}(t)=-\tilde{H}(\tilde{\w}(t))^{-1}\nabla \tilde{F}(\tilde\w(t)),
\end{equation}
where note that $\nabla \tilde{F} (\tilde{\w})=\nabla \reparam^{-1}(\tilde{\w})^\top \nabla F(\reparam^{-1}(\tilde\w))$. 

Since our initial development of the Riemannian Gradient Flow in  eq. \eqref{eq:true-rgf} was independent of the parametrization, it should be the case that the solution $\tilde{\w}(t)$ of \eqref{eq:tilde-rgf}, i.e.~as gradient flow in $(\tilde{\c{W}},\tilde{H})$, is equivalent to gradient flow in $(\c{W},H)$, i.e.~$\tilde\w(t)=\reparam(\w(t))$ where $\w(t)$ is the solution of \eqref{eq:rgf}. It is however, insightful to verify this directly: to do so, let us take the solution $\w(t)$ of \eqref{eq:rgf},  define $\tilde{\w}(t)\doteq\reparam(\w(t))$, and check whether it is in-fact a solution to \eqref{eq:tilde-rgf}.  Starting from the left hand side of \eqref{eq:tilde-rgf}, we have:
\begin{align}
    \dot{\tilde{\w}} &= \nabla \reparam(\w) \dot{\w}=-\nabla \reparam(\w)
    H(\w)^{-1} \nabla F(\w) \notag\\
    &= -\tilde{H}(\tilde\w)^{-1} \nabla \tilde{F}(\tilde\w),
\end{align}
 thus verifying that $\reparam(\w(t))$ indeed satisfies \eqref{eq:tilde-rgf}.

Do the same arguments hold also for the Mirror Descent discretization \eqref{eq:md-disc}?  Taking the solution $\w(t)$ of \eqref{eq:md-disc} and setting $\tilde{\w}\doteq\reparam(\w)$, we can follow the same derivation as above, except now the metric tensor $H$ and gradient $\nabla F$ are calculated at different points, $\w(t)$ and $\wteta$, respectively.
\begin{flalign}
    \dot{\tilde{\w}} &= -\nabla \reparam(\w) 
    H(\w)^{-1} \nabla F(\wteta) \notag\\
    &= {{-\tilde{H}(\tilde\w)^{-1} {\left( \nabla \reparam(\wteta)^{-1} \nabla \reparam(\w)\right)^\top} \nabla \tilde{F}(\tildewteta)}}.\label{eq:MDchange}
\end{flalign}
We can see why the Mirror Descent discretization, and hence also the Mirror Descent iterates are {\em not} invariant to changes in parametrization: if $\nabla \reparam(\w)$ is fixed, i.e.,~the reparametrization is affine, we have  $\nabla \reparam(\wteta)^{-1} \nabla \reparam(\w)=I$ and \eqref{eq:MDchange} shows that $\tilde\w=\reparam(\w)$ satisfies the Mirror Descent discretized differential equation w.r.t.~$\tilde H(\tilde\w)$.  But more generally, the discretization would be affected by the ``alignment'' of the Jacobians along the solution path. We note that, for essentially the same reason, NGD is not generally invariant to reparametrization either.


A related question is how a reparametrization affects whether the metric tensor is a Hessian map.  Indeed, for a particular parametrization (i.e.~chart), the existence of a potential function $\psi$ such that $H=\nabla^2 \psi$ depends on whether $H$ satisfies \eqref{eq:Hijk}, and it may well be the case that $H(\w)$ is not a Hessian map but $\tilde{H}(\tilde{\w})$ is, or visa versa (in fact, in general if $\reparam(\w)$ is non-affine we cannot expect both $H(\w)$ and $\tilde{H}(\tilde{\w})$ to be Hessian maps).  Does every Riemannian manifold have a reparametrization (i.e.~chart) where the metric tensor is a Hessian map, i.e.~which corresponds to ``classical'' Mirror Descent?  \citet{amari2014curvature} showed that while all Riemannian manifolds isomorphic to $\bR^2$ admit a parametrization for which the metric tensor is a Hessian map, this is not true in higher dimensions; even a manifold isomorphic to $\bR^3$ might not admit any parametrization with a Hessian metric tensor.

We see then how our potential-free derivation of Section \ref{sec:mirrorless} can indeed be much more general than the traditional view of Mirror Descent which applies only when the metric tensor is a Hessian map and a potential function exists: for many Riemannian manifolds, there is no parametrization with a Hessian metric tensor, and so it is not possible to define Mirror Descent updates classically such that the Riemannian gradient flow is obtained as their limit.  Yet, the approach of Section \ref{sec:mirrorless} always allows us to do so.  Furthermore, even for manifolds for which there exists a parametrization where the metric tensor is the Hessian of some potential function, our approach allows considering discretizations in other isometric  parametrization.

Finally, in light of our characterization of Mirror Descent, several readers have suspected that Mirror Descent might be equivalent to Riemannian Gradient Descent \citep[see][]{absil2009optimization} using steps that follow geodesics on the manifold (i.e.~using the exponential map retraction). However, the Riemannian Gradient and geodesics are intrinsic, whereas we have shown that Mirror Descent is not.

\section{Summary}
In this paper we presented a ``primal'' derivation of Mirror Descent, based on a discretization of the Riemannian Gradient Flow,  and showed how it can be useful for understanding, thinking about, and potentially analyzing  Mirror Descent, Natural Gradient Descent, and Riemannian Gradient Flow.  We also showed how this view suggests an generalization of (Mirrorless) Mirror Descent to any Riemannian geometry. It is important to identify interesting and useful examples of metric tensors $H$ that are not Hessian maps for which this Mirrorless Mirror Descent perspective can lead to new algorithms and analysis. 

\paragraph{Acknowledgements} We thank Andr\'e Neves for a helpful discussion about Riemannian geometry and for pointing out \cite{amari2014curvature}. This work was supported by NSF-RI 1764032 and was done, in part, while SG and NS were visiting the Simons Institute for the Theory of Computing. BW is supported by a Google Research PhD fellowship.

\section*{Broader Impact}

Mirror Descent and Natural Gradient Descent are important and popular optimization approaches, both theoretically and practically, and both play central roles in machine learning.  Aside from a direct role as a method for minimizing a given optimization objective, Mirror Descent, and ideas derived from it, such as the role of the potential function, also play a central role in online learning \citep[e.g.][for a survey]{shalev2012online}, and throughout learning theory.  Obtaining a better understanding of Mirror Descent, and its relationship with Natural Gradient Descent, has thus been an ongoing endeavour in the optimization and machine learning communities, with past work, e.g.~that of \cite{raskutti2015information}, being influential in guiding the community's thinking about these methods.   There is also been much interest in the community lately in understanding and re-deriving optimization methods as discretizations of continuous solutions to differential equations (e.g., \cite{wibisono2016variational}).  Obtaining a novel, and very different, derivation of Mirror Descent as such a discretization can thus be very impactful in guiding our thinking about it, and in devising novel insights and methods based on it.  
Our novel view could be particularly impactful since unlike all prior derivations of Mirror Descent, our approach does {\em not} rely on a dual and is thus valid much more broadly, and allows generalizing Mirror Descent to many more settings (as discussed in Section \ref{sec:mirrorless}).  

\paragraph{Cross-Disciplinary Impact and Impact in Education}  Beyond the possible practical implications, our primal-only view also has pedagogical implications, as it can allows for an arguably more direct derivation of Mirror Descent that might be easier to understand intuitively, especially by an audience not familiar with duality, link functions and Bregman divergences.  As such, it can open up understanding of this method to a wider audience.  In fact, the derivation was initially derived in order to explain Mirror Descent to physicists, and several colleagues already adopted it in the classroom.

\appendix
\section*{Appendix}

\section{Stochastic Discretization}
\newcommand{\wsmdk}[1]{\hat{\w}_\text{MD}^{(#1)}}
 In this  appendix, we briefly discuss how yet another discretization of Riemannian Gradient Flow captures Stochastic Mirror Descent \citep{nemirovsky1983problem}, and could be useful in studying optimization versus statistical issues in training.
We have so far discussed exact, or {\em batch} Mirror Descent, but a popular variant is {\em Stochastic} Mirror Descent, where at each iteration we update based on an unbiased estimator $g^{(k)}$ of the gradient $\nabla F(\w^{(k)})$, i.e.~such that $\mathbb{E} g^{(k)}=\nabla F(\w^{(k)})$ as, 
\begin{equation}
   \wsmdk{k+1}=\argmin_{{\w}\in\c{W}}\; \eta\innerprod{g^{(k)}}{\w}+ D_\psi(\w,\wsmdk{k}).
\label{eq:md-gen-2}
\end{equation}
Consider stochastic objective of the form:
\begin{equation}
    \label{eq:Fwz}
    F(\w)=\mathbb{E}_{z}f(\w,z).
\end{equation}
We can derive Stochastic Mirror Descent from the following stochastic discretization of  Riemannian Gradient Flow \eqref{eq:rgf}: 
\begin{equation}
    \dot{\w}(t) = -H(\w(t))^{-1}\nabla f(\wteta,\ztnu)
    \label{eq:disc-stoch}
\end{equation}
where $z_t$ are sampled i.i.d., and we used two different resolutions, $\eta$ and $\nu$, to control the discretization.

Setting $\nu=\eta$ and taking $\wsmdk{k}=\w(\eta k)$ we recover ``single example`` Stochastic Mirror Descent, i.e.~where at each iteration we use a gradient estimator $g^{(k)}=\nabla f(\w^{(k)},z)$ based on a single i.i.d.~example.  But varying $\nu$ relative to $\eta$ also allows us to obtain other variants.  

Taking $\nu<\eta$, e.g. $\eta=b\cdot\nu$ for $b>1$, we recover Mini-Batch Stochastic Mirror Descent, where at each iteration we use a gradient estimator obtained by averaging across $b$ i.i.d.~examples.  To see this, note that solving \eqref{eq:disc-stoch} as in Theorem \ref{thm:main} we have that for $i=0,\ldots,b-1$, $\nabla\psi(\w(k\eta+(i+1)\nu))=\nabla\psi(\w(k\eta+i\nu))-\nu \nabla f(\w(k\eta),z_{k\eta+i\nu})$ and so $\nabla\psi(\w((k+1)\eta))=\nabla\psi(\w(k\eta+b\nu))=\nabla\psi(\w(k\eta))-\eta\frac{1}{b}\sum_i \nabla f(\w(k\eta),z_{k\eta+i\nu})$.

At an extreme, as $\nu\rightarrow 0$, the solution of \eqref{eq:disc-stoch} converges to the solution of the  Mirror Descent discretization \eqref{eq:md-disc} and we recover the exact Mirror Descent updates on the population objective.  

It is also interesting to consider $\eta<\nu$, in particular when  $\eta\rightarrow 0$ while  $\nu>0$ is fixed. This corresponds to optimization using stochastic (infinitesimal) gradient flow, where over a time $T$ we use $T/\nu$ samples.  Studying how close the discretization \eqref{eq:disc-stoch} remains  to the population Riemannian Gradient Flow  \eqref{eq:rgf}, in terms of $\eta$ and $\nu$, could allow us to tease apart the optimization complexity and sample complexity of learning (minimizing the population objective).

\section{Convergence of Mirrorless Mirror Descent}\label{app:convergence}
\begin{restatable}{theorem}{thmconvergence}
Let the metric tensor $0 \prec \alpha I \preceq H(\w) \preceq \beta I$ for all $\w$ and let $F$ be $\gamma$-smooth and $\lambda$-strongly convex with respect to L2. Then the updates \eqref{eq:mirrorless} with constant stepsize $\eta = \frac{\alpha^2}{\gamma\beta}$ will converge at a rate
\[
F(\wmdk{K}) - F^* 
\leq \left(F(\wmdk{0}) - F^*\right) \exp\left(- \frac{\lambda\alpha^2 K}{\gamma\beta^2} \right) 
\]
\end{restatable}
\begin{proof}
Throughout this proof, we will use $\|\cdot\|$ exclusively to denote the L2 norm.
We begin by observing that
\begin{align*}
\wmdk{k+1} - \wmdk{k}
&= -\int_{\eta k}^{\eta (k+1)} H(\w(t))^{-1}\nabla F(\wmdk{k}) dt \\
&= -\hat{H}_k \nabla F(\wmdk{k})
\end{align*}
for some matrix $\eta \beta^{-1} I \preceq \hat{H}_k \preceq \eta\alpha^{-1} I$. Furthermore, by the $\gamma$-smoothness of $F$
\begin{align*}
&F(\wmdk{k+1}) - F^* \\
&\leq F(\wmdk{k}) - F^* + \innerprod{\nabla F(\wmdk{k})}{\wmdk{k+1} - \wmdk{k}} \\
&\qquad\qquad+ \frac{\gamma}{2}\nrm*{\wmdk{k+1} - \wmdk{k}}^2 \\
&= F(\wmdk{k}) - F^* - \innerprod{\nabla F(\wmdk{k})}{\hat{H}_k \nabla F(\wmdk{k})} \\
&\qquad\qquad +\frac{\gamma}{2}\nrm*{\hat{H}_k \nabla F(\wmdk{k})}^2 \\
&\leq F(\wmdk{k}) - F^* - \eta\beta^{-1}\nrm*{\nabla F(\wmdk{k})}^2 \\
&\qquad\qquad+ \frac{\eta^2\gamma}{2\alpha^2}\nrm*{\nabla F(\wmdk{k})}^2 \\
&= F(\wmdk{k}) - F^* - \frac{\alpha^2}{2\gamma\beta^2}\nrm*{\nabla F(\wmdk{k})}^2
\end{align*}
where we used that $\eta = \frac{\alpha^2}{\gamma\beta}$ for the final equality. Finally, we note that by the $\lambda$-strong convexity of $F$, for any $\w$
\begin{equation*}
\nrm*{\nabla F(\w)}^2 \geq 2\lambda\left( F(\w) - F^* \right)
\end{equation*}
We conclude that
\begin{align*}
F(\wmdk{k+1}) - F^*
&\leq \left(1 - \frac{\lambda\alpha^2}{\gamma\beta^2} \right) \left(F(\wmdk{k}) - F^*\right) 
\end{align*}
Unrolling this recursion yields the stated bound.
\end{proof}

{
\section{Implications for Implicit Bias}  

In this Appendix we discuss a related issue to the main thrust of the paper, which is the implicit bias of Riemannian Gradient Flow and its relationship to the existence of a potential function such that $\nabla^2 \psi = H$.

In an underdetermined (overparametrized) optimization problem, when there are many global minima, it is important to understand not only that a minimum is reached, but \emph{which} global minimum is reached.  Such  situations are common in contemporary machine learning. When using models with many more parameters then data points, and the implicit bias of the optimization algorithm (i.e.~the bias to prefer some global minimum over others) has been argued as playing an important role, specially in deep learning \citep[e.g.][]{neyshabur2014search,neyshabur2017geometry,
gunasekar2017implicit,telgarsky2013margins,
lyu2019gradient}).  

At least when optimizing overparametrized linear regression objective, we have a precise understand of the implicit bias of Mirror Descent, in terms of the potential $\psi$ and Bregman divergence $D_\psi(\cdot,\cdot)$ \cite{gunasekar2018characterizing}. With any sufficiently small stepsize\footnote{To be more precise, in realizable models if Mirror Descent  converges to a zero error solution--see \citet{gunasekar2018characterizing} for details}, Mirror Descent converges to the special global minimum that minimizes the Bregman divergence to the initialization:
\begin{equation}
    \argmin_{\w} D_\psi(\w,\w(0)) \quad \mathrm{s.t.}\quad F(\w)=0.
\end{equation}
Since this characterization holds for any sufficiently small  stepsize, it also holds for the infinitesimal limit of Mirror Descent, \textit{i.e.,} the Riemannian Gradient Flow \eqref{eq:rgf}.  Hence, for  Riemannian Gradient Flows that  corresponds to a potential $\psi$, we get a crisp characterization of its implicit bias.

\paragraph{Example: Square Parametrization}
As an example,  we discuss a recently studied flow that leads to an interesting implicit bias \cite{woodworth2019kernel}.  Consider the linear regression objective $F(\w)=\norm{A\w-b}^2$, with the dynamics induced on $\w$ by gradient descent dynamics w.r.t.~\,\,$(\bfu_+,\bfu_-)\in\bR^{2d}$ where $\w=\bfu_+^2-\bfu_-^2$ element-wise.  We can think of these as the dynamics on the predictor implemented by a two layer ``diagonal'' linear network when we train the weights by gradient descent.  \citet{woodworth2019kernel} showed that with initialization $\bfu_+(0)=\bfu_-(0)=\alpha \mathbf{1}_{d}$, the induced dynamics on $\w$ are,
\begin{equation}
\dot{\w}(t) = -2\alpha^2\diag\brk*{\sqrt{\frac{\w(t)^2}{4\alpha^4} + 1}} \nabla F(\w(t)).
\end{equation}
This corresponds to Riemannian Gradient Flow on $F$ with the metric tensor
\begin{equation}
H(\w) = \diag\brk*{\sqrt{\w^2 + 4\alpha^4}}^{-1}.
\end{equation}
Since the metric tensor is diagonal, we can integrate each component and get $H(\w)=\nabla^2 \psi_\alpha(\w)$ for the potential:
\begin{equation}
\psi_\alpha(\w) = \sum_{i=1}^d \w_i\arcsinh\prn*{\frac{\w_i}{2\alpha^2}} - \sqrt{\w_i^2 + 4\alpha^4}.
\end{equation}
We can therefore conclude that $\w$ will converge to $\argmin_{\w} D_{\psi_\alpha}(\w(0),\w)\ \textrm{s.t.}\ F(\w) = 0$, recovering the result from \cite{woodworth2019kernel}.  As discussed there, this implicit bias interpolates between the $\ell_1$ norm, which is obtained as the limit when $\alpha\rightarrow 0$ and the $\ell_2$, which is obtained as the limit when $\alpha\rightarrow\infty$.

\paragraph{A Negative Example} To see a situation where a potential does not exist, consider instead a matrix factorization model 
\begin{equation}
F(\W) = \nrm*{\mathcal{A}(\W) - y}^2.
\end{equation}
where $\mathcal{A}: \R^{d\times d} \rightarrow \R^N$ is a linear operator defined as $\mathcal{A}(Y)_i = \tri*{A_i, Y}$ and $\W$ is parametrized as $\W = \U\U^\top$. Then, gradient flow with respect to $\U$ yields dynamics \cite{gunasekar2017implicit}
\begin{equation}
\dot{\W}(t) =  -\nabla_\W F(\W(t)) \W(t) - \W(t) \nabla_\W F(\W(t)).
\end{equation}
The dynamics on $\W$ are again a Riemannian Gradient Flow, this time with the metric tensor $H(\W)$ that is the inverse of the linear map $G \rightarrow \W G + G \W$. Unfortunately, $H(\W)$ does not satisfy \eqref{eq:Hijk}, and is therefore \emph{not} a Hessian map, so it cannot be understood in terms of a potential function $\psi$. Thus, we do not have a crisp characterization of the implicit bias, and we cannot define Mirror Descent updates classically (but we can still define them using the Mirrorless Mirror Descent discretization!).

\paragraph{Reparametrizations} It is important to distinguish between two kinds of reparametrizations. When we parametrize $\w=\bfu_+^2-\bfu_-^2$ or $\W=\U\U^\top$ and then used the Euclidean geometry (corresponding to gradient descent) on $\bfu$, this was not an isometry as we did {\em not}  maintain the manifold structure and metric.  This type of change {\em does} effect the Gradient Flow dynamics, which is how we got interesting dynamics when we mapped back to $\w$.  This is different from the isometric reparametrizations discussed in Section \ref{sec:reparam}.  As we discussed, such isometric reparametrizations do {\em not}  change the Riemannian Gradient Flow dynamics.  Additionally, although Mirror Descent updates (with finite step size) \emph{do} change when we change parametrization (even if we preserve the intrinsic manifold metric, as discussed in Section \ref{sec:reparam}), the implicit bias of Mirror Descent (i.e.~the limit point of Mirror Descent updates), does {\em not} change since it is entirely determined by the (invariant) potential function.
}

\nocite{juditsky2019unifying}
\bibliographystyle{plainnat}
\bibliography{bib}

\begin{thebibliography}{21}
\providecommand{\natexlab}[1]{#1}
\providecommand{\url}[1]{\texttt{#1}}
\expandafter\ifx\csname urlstyle\endcsname\relax
  \providecommand{\doi}[1]{doi: #1}\else
  \providecommand{\doi}{doi: \begingroup \urlstyle{rm}\Url}\fi

\bibitem[Absil et~al.(2009)Absil, Mahony, and Sepulchre]{absil2009optimization}
P-A Absil, Robert Mahony, and Rodolphe Sepulchre.
\newblock \emph{Optimization algorithms on matrix manifolds}.
\newblock Princeton University Press, 2009.

\bibitem[Amari(1998)]{amari1998natural}
Shun-Ichi Amari.
\newblock Natural gradient works efficiently in learning.
\newblock \emph{Neural computation}, 1998.

\bibitem[Amari(2012)]{amari2012differential}
Shun-ichi Amari.
\newblock \emph{Differential-geometrical methods in statistics}.
\newblock Springer Science \& Business Media, 2012.

\bibitem[Amari and Armstrong(2014)]{amari2014curvature}
Shun-ichi Amari and John Armstrong.
\newblock Curvature of hessian manifolds.
\newblock \emph{Differential Geometry and its Applications}, pages 1--12, 2014.

\bibitem[Beck and Teboulle(2003)]{beck2003mirror}
Amir Beck and Marc Teboulle.
\newblock Mirror descent and nonlinear projected subgradient methods for convex
  optimization.
\newblock \emph{Operations Research Letters}, 2003.

\bibitem[Do~Carmo(2016)]{do2016differential}
Manfredo~P Do~Carmo.
\newblock \emph{Differential Geometry of Curves and Surfaces: Revised and
  Updated Second Edition}.
\newblock Courier Dover Publications, 2016.

\bibitem[Gunasekar et~al.(2017)Gunasekar, Woodworth, Bhojanapalli, Neyshabur,
  and Srebro]{gunasekar2017implicit}
Suriya Gunasekar, Blake~E Woodworth, Srinadh Bhojanapalli, Behnam Neyshabur,
  and Nati Srebro.
\newblock Implicit regularization in matrix factorization.
\newblock In \emph{Advances in Neural Information Processing Systems}, pages
  6151--6159, 2017.

\bibitem[Gunasekar et~al.(2018)Gunasekar, Lee, Soudry, and
  Srebro]{gunasekar2018characterizing}
Suriya Gunasekar, Jason Lee, Daniel Soudry, and Nathan Srebro.
\newblock Characterizing implicit bias in terms of optimization geometry.
\newblock In \emph{International Conference on Machine Learning}, pages
  1832--1841, 2018.

\bibitem[Juditsky et~al.(2019)Juditsky, Kwon, and
  Moulines]{juditsky2019unifying}
Anatoli Juditsky, Joon Kwon, and {\'E}ric Moulines.
\newblock Unifying mirror descent and dual averaging.
\newblock \emph{arXiv preprint arXiv:1910.13742}, 2019.

\bibitem[Lyu and Li(2020)]{lyu2019gradient}
Kaifeng Lyu and Jian Li.
\newblock Gradient descent maximizes the margin of homogeneous neural networks.
\newblock \emph{International Conference on Learning Representations (ICLR)},
  2020.

\bibitem[Nemirovsky and Yudin(1983)]{nemirovsky1983problem}
Arkadii~Semenovich Nemirovsky and David~Borisovich Yudin.
\newblock Problem complexity and method efficiency in optimization.
\newblock 1983.

\bibitem[Neyshabur et~al.(2015)Neyshabur, Tomioka, and
  Srebro]{neyshabur2014search}
Behnam Neyshabur, Ryota Tomioka, and Nathan Srebro.
\newblock In search of the real inductive bias: On the role of implicit
  regularization in deep learning.
\newblock \emph{International Conference on Learning Representations (ICLR)
  workshop track}, 2015.

\bibitem[Neyshabur et~al.(2017)Neyshabur, Tomioka, Salakhutdinov, and
  Srebro]{neyshabur2017geometry}
Behnam Neyshabur, Ryota Tomioka, Ruslan Salakhutdinov, and Nathan Srebro.
\newblock Geometry of optimization and implicit regularization in deep
  learning.
\newblock \emph{arXiv preprint arXiv:1705.03071}, 2017.

\bibitem[Raginsky and Bouvrie(2012)]{raginsky2012continuous}
Maxim Raginsky and Jake Bouvrie.
\newblock Continuous-time stochastic mirror descent on a network: Variance
  reduction, consensus, convergence.
\newblock In \emph{IEEE Conference on Decision and Control (CDC)}. IEEE, 2012.

\bibitem[Raskutti and Mukherjee(2015)]{raskutti2015information}
Garvesh Raskutti and Sayan Mukherjee.
\newblock The information geometry of mirror descent.
\newblock \emph{IEEE Transactions on Information Theory}, 2015.

\bibitem[Shalev-Shwartz(2012)]{shalev2012online}
Shai Shalev-Shwartz.
\newblock Online learning and online convex optimization.
\newblock \emph{Foundations and Trends{\textregistered} in Machine Learning},
  pages 107--194, 2012.

\bibitem[Srebro et~al.(2011)Srebro, Sridharan, and
  Tewari]{srebro2011universality}
Nati Srebro, Karthik Sridharan, and Ambuj Tewari.
\newblock On the universality of online mirror descent.
\newblock In \emph{Advances in neural information processing systems}, pages
  2645--2653, 2011.

\bibitem[Telgarsky(2013)]{telgarsky2013margins}
Matus Telgarsky.
\newblock Margins, shrinkage, and boosting.
\newblock In \emph{International Conference on Machine Learning}, pages
  307--315, 2013.

\bibitem[Warmuth and Jagota(1997)]{warmuth1997continuous}
Manfred~K Warmuth and Arun~K Jagota.
\newblock Continuous and discrete-time nonlinear gradient descent: Relative
  loss bounds and convergence.
\newblock In \emph{International Symposium on Artificial Intelligence and
  Mathematics}, 1997.

\bibitem[Wibisono et~al.(2016)Wibisono, Wilson, and
  Jordan]{wibisono2016variational}
Andre Wibisono, Ashia~C Wilson, and Michael~I Jordan.
\newblock A variational perspective on accelerated methods in optimization.
\newblock \emph{Proceedings of the National Academy of Sciences}, pages
  E7351--E7358, 2016.

\bibitem[Woodworth et~al.(2020)Woodworth, Gunasekar, Lee, Soudry, and
  Srebro]{woodworth2019kernel}
Blake Woodworth, Suriya Gunasekar, Jason Lee, Daniel Soudry, and Nathan Srebro.
\newblock Kernel and deep regimes in overparametrized models.
\newblock \emph{Conference on Learning Theory (COLT)}, 2020.

\end{thebibliography}

\end{document}